\documentclass[conference]{IEEEtran}
\IEEEoverridecommandlockouts
\usepackage{cite}
\usepackage{amsmath,amssymb,amsfonts}
\usepackage{graphicx}
\usepackage{textcomp}
\usepackage{xcolor}
\usepackage{amsfonts}
\usepackage{amsmath}
\usepackage{amssymb}
\usepackage{subfig}
\usepackage{caption}
\usepackage{mathrsfs}
\usepackage{graphicx}
\usepackage{graphics}

\usepackage{amsthm}
\usepackage{xcolor}
\usepackage{yhmath}
\usepackage{float}
\usepackage{algorithm}
\usepackage{algpseudocode}
\usepackage{epsfig}
\usepackage{amsmath,bm}
\usepackage{cite}
\usepackage{graphicx,color}
\usepackage{multicol}
\usepackage{booktabs}
\usepackage{graphicx}

\def\BibTeX{{\rm B\kern-.05em{\sc i\kern-.025em b}\kern-.08em
    T\kern-.1667em\lower.7ex\hbox{E}\kern-.125emX}}

\newtheorem{theorem}{Theorem}

\begin{document}

\title{Foresight of Graph Reinforcement Learning: Latent Permutations Learnt by Gumbel Sinkhorn Network\\
}

\author{\IEEEauthorblockN{Tianqi Shen}
\IEEEauthorblockA{\textit{Department of AeroSpace Information Engineering} \\
\textit{School of Aerospace Engineering, Beihang University}\\
Beijing, China \\
tqshen@buaa.edu.cn}
\and
\IEEEauthorblockN{Hong Zhang}
\IEEEauthorblockA{\textit{Department of AeroSpace Information Engineering} \\
\textit{School of Aerospace Engineering, Beihang University}\\
Beijing, China \\
dmrzhang@buaa.edu.cn}
\and
\IEEEauthorblockN{Ding Yuan}
\IEEEauthorblockA{\textit{Department of AeroSpace Information Engineering} \\
\textit{School of Aerospace Engineering, Beihang University}\\
Beijing, China \\
dyuan@buaa.edu.cn}
\and
\IEEEauthorblockN{Jiaping Xiao}
\IEEEauthorblockA{\textit{School of Mechanical and Aerospace Engineering} \\
\textit{Nanyang Technological University}\\
Singapore, Singapore \\
jiaping001@e.ntu.edu.sg}
\and
\IEEEauthorblockN{Yifan Yang}
\IEEEauthorblockA{\textit{Institute of Artificial Intelligence} \\
\textit{Beihang University}\\
Beijing, China \\
stephenyoung@buaa.edu.cn}
}

\maketitle

\begin{abstract}
Vital importance has necessity to be attached to cooperation in multi-agent environments, as a result of which some reinforcement learning algorithms combined with graph neural networks have been proposed to understand the mutual interplay between agents. However, highly complicated and dynamic multi-agent environments require more ingenious graph neural networks, which can comprehensively represent not only the graph topology structure but also evolution process of the structure due to agents emerging, disappearing and moving. To tackle these difficulties, we propose Gumbel Sinkhorn graph attention reinforcement learning, where a graph attention network highly represents the underlying graph topology structure of the multi-agent environment, and can adapt to the dynamic topology structure of graph better with the help of Gumbel Sinkhorn network by learning latent permutations. Empirically, simulation results show how our proposed graph reinforcement learning methodology outperforms existing methods in the PettingZoo multi-agent environment by learning latent permutations.
\end{abstract}


\section{Introduction}
Graph structure has tremendous potential and is crucially important to enable agents to comprehend cooperation in multi-agent environments for many applications, such as flight formation\cite{khan2020graph}, autonomous driving\cite{gammelli2021graph} and signal control\cite{nishi2018traffic}, as a result of which multi-agent reinforcement learning(MARL) facilitated by graph neural networks have been widely exploited. Because each agent is a node of the graph, graph neural network has the ability to gather information from farther agents while direct communication occurs within only nearby agents. However, as shown in Fig. \ref{fig1}, existing graph reinforcement learning methods often trap into two dilemmas. Firstly, to gain more valuable information and more mutual cooperation requires more ingenious graph neural networks. Secondly, the critical omission in existing graph reinforcement learning methods is that the graph at the next time step can hardly be efficiently predicted from only current time step. It is imaginable that the prediction is significant to improve the robustness of graph reinforcement learning when dealing with high dynamic multi-agent environments.
\begin{figure}[htbp]
\centerline{\includegraphics[width=6.944cm, height=7.364cm]{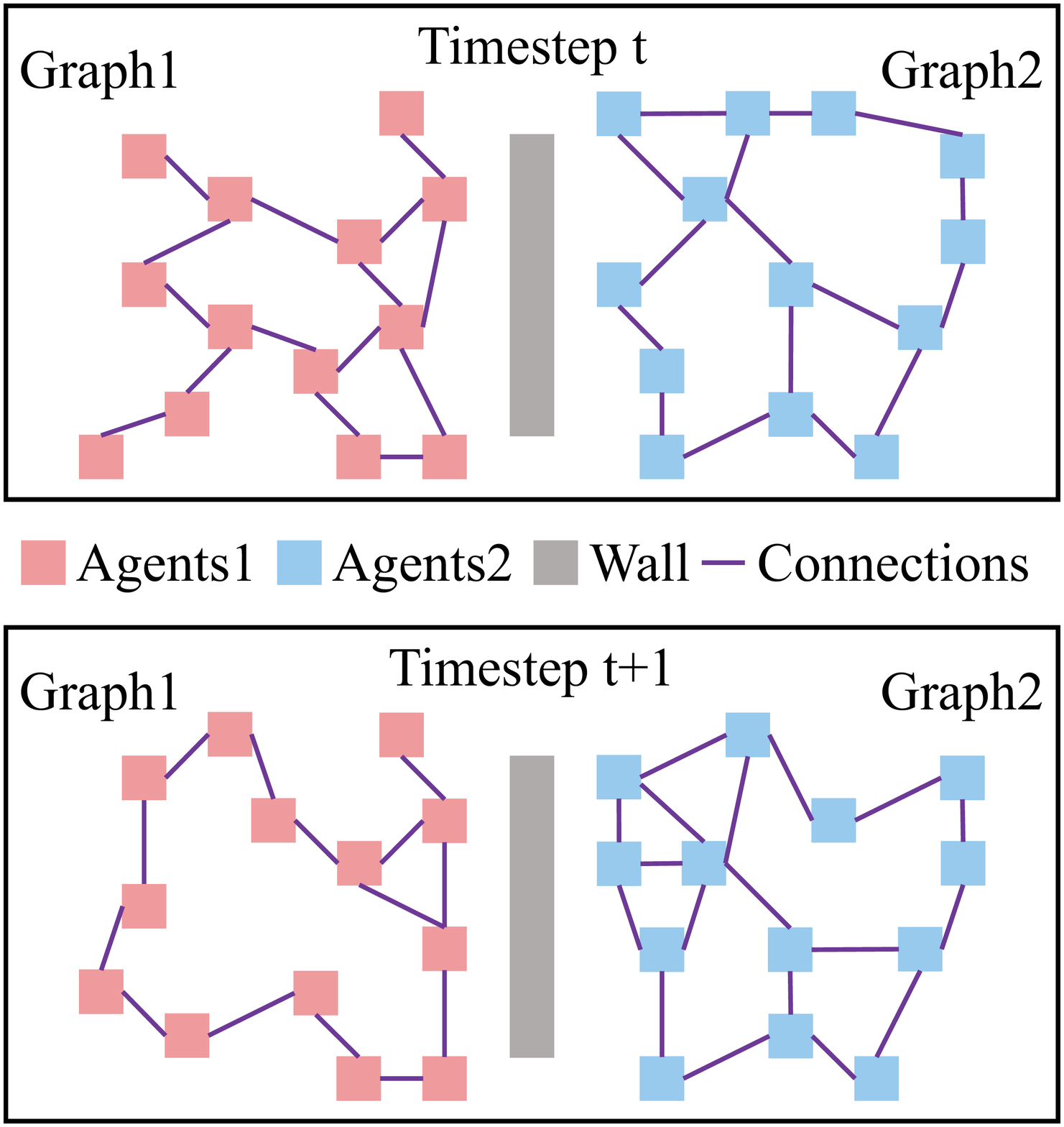}}
\caption{Dynamic graph structure in multi-agent environment.}
\label{fig1}
\end{figure}
In this paper, Gumbel Sinkhorn graph attention reinforcement learning (GS-GAT) is proposed to enhance mutual cooperation and adapt dynamic change in multi-agent environments. GS-GAT is instantiated based on the deep Q network, namely DQN\cite{mnih2015human}, and end-to-end trained. Parameters in the model are shared among all agents to making it easy to scale. The off-policy of DQN requires two networks to predict rewards at current time step and those at next time step respectively when taking actions. The former network is named as local network and the latter one as target network. Both the two networks are constructed by a graph attention network, namely GAT\cite{velivckovic2017graph}. Graph, which is highly represented and learnt by GAT in this paper, is modeled in following steps: 1) each agent acts as a node; 2) edges are built between each agent and its $|\mathcal{B}|$ nearest neighbours; 3) feature of each node is the embedding of local observations of the agent. Multi-head attention\cite{vaswani2017attention} utilized in GAT convolves the features from neighboring nodes with receptive fields gradually increasing. This mechanism is exploited to learn cooperative policies. Assistance is obtained from Gumbel Sinkhorn network\cite{mena2018learning} to enable GAT adapting dynamic graph structure. Latent permutations, which represent the transition from output of local network to that of target network, can be learnt by Gumbel Sinkhorn network and employed as foresight of agents to predict graph topology structure at next time step. We empirically show the learning effectiveness of GS-GAT in Gather game and Battle game in MAgent environment\cite{zheng2018magent}, which is wrapped in PettingZoo environment\cite{terry2020pettingzoo} now. Demonstration is presented that GS-GAT outperforms other graph reinforcement learning methods in a large margin. By ablation studies, advantage of application of Gumbel Sinkhorn network to predict graph topology structure is validated.

\section{Related Work}
\subsection{Multi-Agent Reinforcement Learning(MARL)}
Improvement of multi-agent cooperation were previously considered by sharing policy parameters in PS-TRPO\cite{gupta2017cooperative}, but this sharing was not an efficient way of information sharing among agents. Networked multi-agent MDP\cite{zhang2018fully} proposed sharing parameters of value function rather than policy with the convergence guarantee of linear function approximation. However, this sharing method and its convergence is base on the assumption of fully observable environments. Value propagation is proposed\cite{qu2019value} for networked MARL, which uses softmax temporal consistency to connect value and policy updates. However, this method only works on networked agents with static connectivity. To deal with complicatedly dynamic structure of a graph, graph convolution networks and permutation equivariance of graph convolution are introduced into information propagation\cite{khan2020graph}. However, the permutation equivariance can be easily broken with the change of graph topology structure and the graph convolution network is not expressive for graph structure. DGN\cite{jiang2018graph} constructs graph structure more complex by using relation kernel built by multi-head attention mechanism. However, foresight of agents in multi-agent environments is still not fulfilled.

\subsection{Graph neural network(GNN)}
Graph structure implicitly underlie many applications in real world, such as drug-drug interaction\cite{lin2020kgnn}, multiple object tracking\cite{jiang2019graph} and semantic segmentation\cite{qi20173d}. Several frameworks have been architected to extract locally connected features from arbitrary graphs. A graph convolutional network, namely GCN\cite{niepert2016learning}, takes as inputs the feature matrix that summarizes the attributes of each node and outputs a node-level feature matrix, this process is called aggregation. The function is similar to the convolution operation in CNN, where the kernels are convolved across local regions of the input to produce feature maps. However, expressive ability of GCN is not strong enough to learn these feature maps. GAT introduced multi-head attention mechanism into graph structure and complicate the process of aggregation. One of most important properties of aggregation operation is permutation equivariance\cite{gama2020stability}, which helps simplify the computation of graph network output if graph topology structure is unchangeable. However, permutation equivariance will easily be broken with the change of graph topology structure. Gumbel Sinkhorn network\cite{mena2018learning} proposed a method to learn latent permutation, which can be utilized to rebuild the broken permutation equivariance.

\section{Method}
Graph used in reinforcement learning is built depending on the observations of current multi-agent environments. Each agent is a node of the graph and the edges between them represent the distance of one certain agent and its neighbours. The team is composed of $N$ agents generically indexed by $n$ which at any given point in time $t$ gain a observation $\boldsymbol{o}_{n}^t\in\mathcal{O}$ in configuration space and must choose an action $\boldsymbol{a}_{n}^t\in\mathcal{A}$ in action space. The observation will be preprocessed and embedded into a feature $\boldsymbol{x}_{n}^t\in\mathcal{X}$.

\subsection{Graph Attention Networks}
This paper represents and learns graph structure using graph attention network. Consider a graph $\mathcal{G}_t=(\boldsymbol{V}_t,\boldsymbol{E}_t)$ at timestep $t$ described by a set of $N_t$ nodes denoted $\boldsymbol{V}_t$, and a set of edges denoted $\boldsymbol{E}_t\subseteq\boldsymbol{V}_t\times\boldsymbol{V}_t$. This graph is considered as the support for a data signal $\boldsymbol{x}_t=\left[x_1^t,\cdots,x_N^t\right]^T$ where the value $x_n^t$ is the feature belong to node $n$. The relation between $\boldsymbol{x}_t$ and $\mathcal{G}_t$ is given by a matrix $\boldsymbol{S}_t$ called the graph shift operator\cite{chung1997spectral}. The elements of $\boldsymbol{S}_t$ given as $s_{ij}^t$ represent the sparsity of the graph, i.e.\ $s_{ij}^t=0$, $\forall i\not=j$ and $\left(i,j\right)\notin\boldsymbol{E}_t$.
Valid examples for graph shift operator are the adjacency matrix, the graph laplacian, and the random walk matrix.

$\boldsymbol{S}_t$ in graph attention network defines a map $\boldsymbol{y}_t=\boldsymbol{S}_t\boldsymbol{x}_t$ between graph signals that represents local exchange of information between a node and its one-hop neighbours. More concretely, if the set of neighbours of node $n$ at timestep $t$ is given by $\mathcal{B}_n^t$ then:
\begin{equation}
y_n^t=[\boldsymbol{S}_t\boldsymbol{x}_t]_n=\sum_{j=n,j\in\mathcal{B}_n^t}s_{nj}^tx_n^t
\end{equation}
performs an aggregation of data at node $n$ from its neighbours that are one-hop away. The aggregation of data at all nodes in graph is denoted $\boldsymbol{y}_t=\left[y_1^t,\cdots,y_N^t\right]$. By repeating this operation, one can access information from nodes located further away. Now one can define the spectral $K$-localized graph convolution as\cite{bruna2013spectral}:
\begin{equation}
\boldsymbol{z}_t=\sum_{k=0}^{K}{h_k\boldsymbol{S}_t^k}\boldsymbol{x}_t=\boldsymbol{H}\left(\boldsymbol{S}_t\right)\boldsymbol{x}_t
\end{equation}
where $\boldsymbol{H}\left(\boldsymbol{S}_t\right)=\sum_{k=0}^{\infty}{h_k\boldsymbol{S}_t^k}$ is a linear shift invariant graph filter with coefficients $h_k$. Similar to CNN the output of GCN is fed into a pointwise non-linear function. Thus, the final form of the graph convolution is given as:
\begin{equation}
\boldsymbol{z}_t=\sigma\left(\boldsymbol{H}\left(\boldsymbol{S}_t\right)\boldsymbol{x}_t\right)
\end{equation}

From the spatial perspective, $\boldsymbol{S}_t$ in graph attention network is an ingenious aggregation method. One certain node's neighbours contribute information of different significance, as a result of which graph attention networks take advantage of attention mechanism to compute the different weights of information coming from one certain node's neighbours. Assume that node $j$ and node $k$ are neighbours of node $i$. The attention weight between node $i$ and node $j$ is:
\begin{equation}
a_{ij}^m=\frac{\exp\left(\boldsymbol{W}_Q^mx_i^t\cdot\left(\boldsymbol{W}_K^mx_j^t\right)^T\right)}{\sum_{k\in\mathcal{B}_i}\exp\left(\boldsymbol{W}_Q^mx_i^t\cdot\left(\boldsymbol{W}_K^mx_k^t\right)^T\right)}
\end{equation}
where $m$ means the $m$-th head of the multi-heads attention mechanism and $\boldsymbol{W}_Q^m$, $\boldsymbol{W}_K^m$ are parameter matrices to be learnt. It is obvious that the graph shift operator of graph attention network is not symmetric. The attention from node $i$ to node $j$ is not equal to that from node $j$ to node $i$ because of $\mathcal{B}_i$ is different from $\mathcal{B}_j$. However, due to this characteristic, graph attention network can utilize the information more efficiently by giving full consideration to each agent's neighbours.

\subsection{Permutation Equivariance}
To control $n$ agents, this paper proposes a graph where each robot is a node and the neighbor agents to the certain agent are connected by an edge linked to the certain agent. The agents are all initialized with random policies and by exploring different actions and then they start to learn what policies best optimize the global reward. Such exploration can change the ordering of the agents and the node features. If the graph topology structure is time invariant, a key property\cite{gama2020stability}, what is called permutation equivariance for graph convolutional filters, is also applicable to graph attention networks.

Given a set of permutation matrices:
\begin{equation}
\boldsymbol{\mathcal{P}}=\left\lbrace\boldsymbol{P}\in{\left\lbrace0,1\right\rbrace}^{N\times N}:\boldsymbol{P}\boldsymbol{1}=\boldsymbol{1},\boldsymbol{P}^T\boldsymbol{1}=\boldsymbol{1}\right\rbrace
\end{equation}
where the operation $\boldsymbol{P}\boldsymbol{x}$ permutes the elements of the vector $\boldsymbol x$. It can be shown that: \\
\begin{theorem}\label{theorem1}
Let graph $\mathcal{G}_t=(\boldsymbol{V}_t,\boldsymbol{E}_t)$ be defined with a graph shift operator $\boldsymbol{S}_t$ of graph attention network at timestep $t$. Further, define $\mathcal{G}_{t+1}$ to be the permuted graph with $\boldsymbol{S}_{t+1}={\boldsymbol P}^T{\boldsymbol{S}_t}{\boldsymbol P}$ for ${\boldsymbol P}\in\boldsymbol{\mathcal{P}}$ and any $\boldsymbol{x}_t \in \mathbb{R}^N$ it holds that:
\begin{equation}
\boldsymbol{H}\left(\boldsymbol{S}_{t+1}\right){\boldsymbol P}^T{\boldsymbol{x}_t}={\boldsymbol P}^T\boldsymbol{H}\left(\boldsymbol{S}_t\right){\boldsymbol{x}_t}
\end{equation}
\end{theorem}
\begin{proof}
Given that $\boldsymbol P$ is a permutation matrix. This implies $\boldsymbol P$ is also an orthogonal matrix, i.e., $\boldsymbol{P}^T\boldsymbol{P}=\boldsymbol{P}\boldsymbol{P}^T=\boldsymbol I$. Thus,
\begin{equation}
\boldsymbol{S}_{t+1}^k=\boldsymbol{P}^T\boldsymbol{S}_t^k\boldsymbol{P}
\end{equation}
Then,
\begin{equation}
\begin{split}
\boldsymbol{H}\left(\boldsymbol{S}_{t+1}\right){\boldsymbol P}^T{\boldsymbol{x}_t}=\sum_{k=0}^{\infty}{h_k\boldsymbol{S}_{t+1}^k}{\boldsymbol P}^T{\boldsymbol{x}_t} \\
=\sum_{k=0}^{\infty}{h_k\left(\boldsymbol{P}^T\boldsymbol{S}_t^k\boldsymbol{P}\right)}{\boldsymbol P}^T{\boldsymbol{x}_t}={\boldsymbol P}^T\boldsymbol{H}\left(\boldsymbol{S}_t\right)\boldsymbol{x}_t
\end{split}
\end{equation}
\end{proof}

Due to the emerging, disappearing and moving of agents which broke the graph topology structure, the graph shift operator $\boldsymbol{S}_t$ of current time step $t$ steps and the one $\boldsymbol{S}_{t+1}$ will lose their connectivity built on permutation matrix $\boldsymbol{P}$, which means:
\begin{equation}
\boldsymbol{S}_{t+1}\neq\boldsymbol{P}^T\boldsymbol{S}_t\boldsymbol{P}
\end{equation}
This will cause
\begin{equation}
\boldsymbol{H}\left(\boldsymbol{S}_{t+1}\right){\boldsymbol{x}_{t+1}}\neq{\boldsymbol P}^T\boldsymbol{H}\left(\boldsymbol{S}_t\right)\boldsymbol{x}_t
\end{equation}
However, the learning of latent permutation matrix $\boldsymbol{P}$ between two timesteps will solve this problem:
\begin{equation}
\boldsymbol{H}\left(\boldsymbol{S}_{t+1}\right){\boldsymbol{x}_{t+1}}={\boldsymbol P}_{\theta}\boldsymbol{W}_p\boldsymbol{H}\left(\boldsymbol{S}_t\right)\boldsymbol{x}_t
\end{equation}

\subsection{Gumbel Sinkhorn Network}
\begin{figure}[htbp]
\centerline{\includegraphics[width=6.846cm, height=5.428cm]{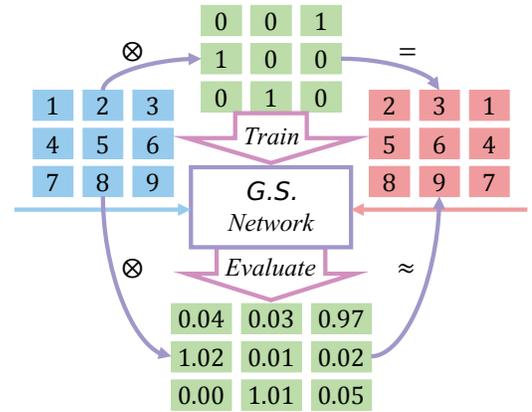}}
\caption{Gumbel Sinkhorn Network.}
\label{fig2}
\end{figure}
Gumbel Sinkhorn is an analog of the Gumbel Softmax distribution for permutations. For the above reason, Gumbel Sinkhorn Network enables learning of the latent permutation matrix referred in the previous content, where is illustrated in Fig. \ref{fig2}.
The non-differentiable parameterization of a permutation can be approximated in terms of a differentiable relaxation, the so-called Sinkhorn operator\cite{adams2011ranking}. The Sinkhorn operator $S(\boldsymbol{X})$ over an $N$ dimensional square matrix $\boldsymbol{X}$ can be expressed as:
\begin{align}
  &S^0(\boldsymbol{X})=\exp(\boldsymbol{X}),S^l(\boldsymbol{X})=\mathcal{T}_c(\mathcal{T}_r(S^{l-1}(\boldsymbol{X})))
\\&S(\boldsymbol{X})=\lim_{l\to \infty}S^l(\boldsymbol{X})
\end{align}

where $\mathcal{T}_r(\boldsymbol{X})=\boldsymbol{X}\oslash(\boldsymbol{X}\boldsymbol{1}_N\boldsymbol{1}_N^T)$ and $\mathcal{T}_c(\boldsymbol{X})=\boldsymbol{X}\oslash(\boldsymbol{1}_N\boldsymbol{1}_N^T\boldsymbol{X})$ as the row and column-wise normalization operators of a matrix, with $\oslash$ denoting the element-wise division and $\boldsymbol{1}_N$ a column vector of ones. The choice of a permutation $\boldsymbol{P}$ can be parameterized through a square matrix $\boldsymbol{X}$, as the solution to the linear assignment problem\cite{kuhn1955hungarian}, with $\boldsymbol{\mathcal{P}}$ denoting the set of permutation matrices and
$\left\langle\boldsymbol{A},\boldsymbol{B}\right\rangle_F=trace(\boldsymbol{A}^T\boldsymbol{B})$ the Frobenius inner product of matrices:
\begin{equation}
M(\boldsymbol{X})=\mathop{\arg\max}_{\boldsymbol{P}\in\boldsymbol{\mathcal{P}}}{\left\langle\boldsymbol{P},\boldsymbol{X}\right\rangle_F}
\end{equation}
where $M(\cdot)$ is called the matching operator, through which we can parameterize the hard choice of a permutation. The difficulty of learning latent permutation is that the $\mathop{\arg\max}$ operator is non-differentiable. However, it has been proved\cite{khan2020graph} that $M(\boldsymbol{X})$ can be obtained as the limit of $S(\boldsymbol{X}/\tau)$, which means:
\begin{equation}
M(\boldsymbol{X})=\lim_{\tau\to0^+}{S(\boldsymbol{X}/\tau)}
\end{equation}
and this theory make the learning of latent permutation differentiable. Based on this conclusion, permutation matrix $\boldsymbol{P}$ can be assumed to follow the Gumbel Sinkhorn distribution with parameter $\boldsymbol{X}$ and temperature $\tau$, denoted $\boldsymbol{P}\backsim\mathcal{G}.\mathcal{S}.(\boldsymbol{X},\tau)$, if it has distribution of $S((\boldsymbol{X}+\epsilon)/\tau)$, where $\epsilon$ is a matrix of standard i.i.d.\ Gumbel noise.

\subsection{Gumbel Sinkhorn Graph Attention Reinforcement Learning}
\begin{figure}[htbp]
\centerline{\includegraphics[width=8.61cm, height=5.236cm]{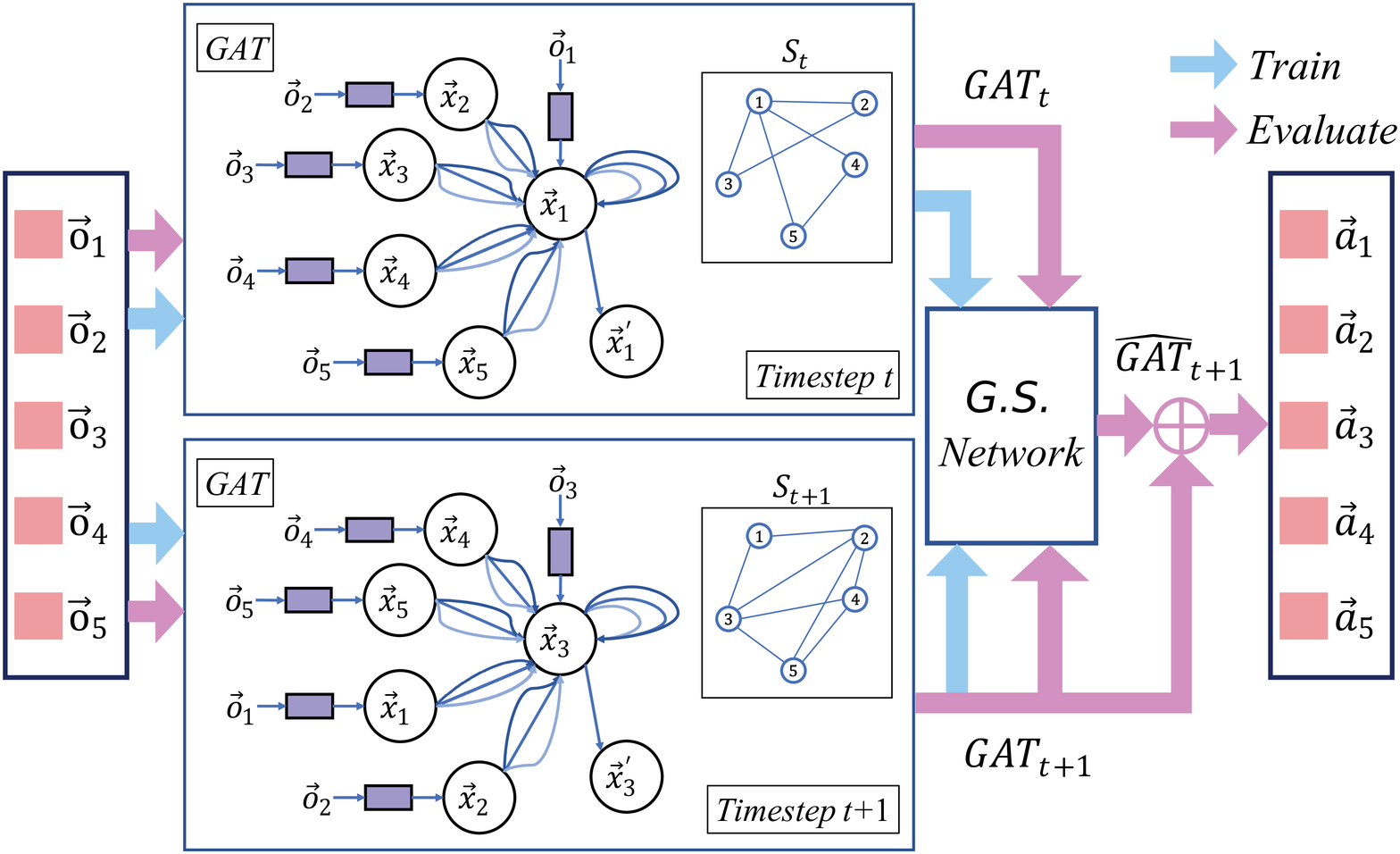}}
\caption{Gumbel Sinkhorn graph attention reinforcement learning.}
\label{fig3}
\end{figure}
Multi-agent reinforcement learning proposed by this paper, which is illustrated with Fig. \ref{fig3}, is developed based on deep Q learning structure, combined with graph attention network and Gumbel Sinkhorn network. Because of the off-policy and time difference in DQN structure, one of the key reinforcement learning trick in DQN is that two different networks are devised to fit the action-value function. Consider two continuous timesteps:
\begin{equation}
\left\langle\boldsymbol{o}_t(\boldsymbol{x}_t),\boldsymbol{a}_t,\boldsymbol{r}_t,\boldsymbol{o}_{t+1}(\boldsymbol{x}_{t+1}),\boldsymbol{a}_{t+1}\boldsymbol{r}_{t+1}\right\rangle
\end{equation}

It means that at timestep $t$ the observation all agents get is $\boldsymbol{o}_t$. The observation is preprocessed and embedded into feature $\boldsymbol{x}_t$. The feature $\boldsymbol{x}_t$ is input into a network named as local network while the output of the network is values of all possible actions. Epsilon greedy strategy\cite{wunder2010classes} is applied to choose an action $\boldsymbol{a}_t$ and the multi-agent environment will feedback reward $\boldsymbol{r}_t$. Difference in the next timestep $t+1$ is that another network named as target network is used to compute values of actions and adopt greedy strategy to take actions.
Both local network and target network use graph attention network, which are denoted as $GAT(\boldsymbol{x}_t,\boldsymbol{a}_t;\theta)$ and $GAT(\boldsymbol{x}_{t+1},\boldsymbol{a}_{t+1};\theta^-)$ respectively. A random minibatch of size $\mathcal{S}$ is sampled from the replay buffer to minimize the loss:
\begin{equation}
\mathcal{L}(\theta)=\frac{1}{\mathcal{S}}\sum_{\mathcal{S}}{\frac{1}{N}\sum_{i=1}^N(TD_{target}-TD_{local})^2}
\end{equation}

\begin{equation}
TD_{target}=r_i^t+\gamma\max_{a_i^t}GAT(\boldsymbol{x}_{t+1},\boldsymbol{a}_{t+1};\theta^-)
\end{equation}

\begin{equation}
TD_{local}=GAT(\boldsymbol{x}_t,\boldsymbol{a}_t;\theta)
\end{equation}
where $TD_{target}$ and $TD_{local}$ indicate concept of time difference in Q-learning algorithm. $\theta^-$ and $\theta$ are parameters of target network and local network respectively. To simplify the formulation, abbreviations are adopted as follows:
\begin{equation}GAT_{t+1}=GAT(\boldsymbol{x}_{t+1},\boldsymbol{a}_{t+1};\theta^-)\end{equation}
\begin{equation}GAT_{t}=GAT(\boldsymbol{x}_t,\boldsymbol{a}_t;\theta)\end{equation}

However, given the broken permutation equivariance in graph attention network discussed above, Gumbel Sinkhorn network is utilized to predict the graph topology structure change. The output of local network is actually $\boldsymbol{H}\left(\boldsymbol{S}_t\right){\boldsymbol{x}_t}$ while the one of target network is actually $\boldsymbol{H}\left(\boldsymbol{S}_{t+1}\right){\boldsymbol{x}_{t+1}}$. To rebuild the broken permutation equivariance is to learn an embedding matrix $\boldsymbol{W}_p$ and a permutation matrix ${\boldsymbol P}_{\theta}$ which satisfy the equation as follows:
\begin{equation}
\boldsymbol{H}\left(\boldsymbol{S}_{t+1}\right){\boldsymbol{x}_{t+1}}={\boldsymbol P}_{\theta}\boldsymbol{W}_p\boldsymbol{H}\left(\boldsymbol{S}_t\right)\boldsymbol{x}_t
\end{equation}

Assume that the Gumbel Sinkhorn network is $\mathcal{G}.\mathcal{S}.(\cdot)$ which takes in two matrices and output the permutation matrix ${\boldsymbol P}_{\theta}$.

Finally, the totally loss $\mathcal{L}(\theta)$ is:
\begin{equation}
\mathcal{L}(\theta)=\frac{1}{\mathcal{S}}\sum_{\mathcal{S}}{\frac{1}{N}\sum_{i=1}^N(\mathcal{L}_{TD}(\theta)+\mathcal{L}_{\mathcal{G}.\mathcal{S}.}(\theta))^2}
\end{equation}
where $\mathcal{L}_{TD}(\theta)$ is the loss to narrow time difference and $\mathcal{L}_{\mathcal{G}.\mathcal{S}.}$ is the loss of Gumbel Sinkhorn network.

$\mathcal{L}_{TD}(\theta)$ is the difference of $TD_{target}$ and $TD_{local}$, which means that:
\begin{equation}
\mathcal{L}_{TD}(\theta)=TD_{target}-TD_{local}
\end{equation}
\begin{equation}
TD_{target}=r_i^t+\gamma\max_{a_i^t}(\alpha GAT_{t+1}+\beta \hat{GAT}_{t+1})
\end{equation}
\begin{equation}TD_{local}=GAT_{t}\end{equation}
where $GAT_{t+1}$ is the actual output of graph attention network at timestep $t+1$. $\hat{GAT}_{t+1}$ is the predicted output of graph attention network at timestep $t+1$, which is computed according to the actual output of graph attention network at timestep $t$ and Gumbel Sinkhorn network. To balance the actual and predicted output, two adjustable hyper-parameters $\alpha$ and $\beta$ are set in the formulation. $\hat{GAT}_{t+1}$ is formulated as follows:
\begin{equation}
\begin{split}
\hat{GAT}_{t+1}&={\boldsymbol P}_{\theta}\boldsymbol{W}_pGAT_{t} \\
&=\mathcal{G}.\mathcal{S}.(GAT_{t}, GAT_{t+1})\boldsymbol{W}_pGAT_{t}
\end{split}
\end{equation}

$\mathcal{L}_{\mathcal{G}.\mathcal{S}.}(\theta)$ is the difference of actual and predicted output of graph attention network at timestep $t+1$, which means that:
\begin{equation}\mathcal{L}_{\mathcal{G}.\mathcal{S}.}(\theta)=\hat{GAT}_{t+1}-GAT_{t+1}
\end{equation}

\section{Experiments}
Experiments have been carried out in a grid-world platform MAgent which is wrapped in the PettingZoo. In the $30\times 30$ grid-world environment, each agent corresponds to one grid and has a local observation. Agents can either move or attack the enemy at each turn. An attack against another agent on their own team will not be registered. Cooperation is investigated among agents in two scenarios, Gather and Battle. Algorithms compared in the experiments are list in TABLE \ref{tab1}.
\begin{table}[htbp]
\caption{Algorithms compared in the experiments}
\begin{center}
\begin{tabular}{c|c|c|c}
\hline
\textbf{order}&\textbf{original al}&\textbf{use $\mathcal{G}.\mathcal{S}.$ network}&\textbf{name} \\
\hline
1&DGN&no&DGN \\
\cline{1-4}
2&DGN&yes&GS-DGN \\
\cline{1-4}
3&GCN&no&GCN \\
\cline{1-4}
4&GCN&yes&GS-GCN \\
\cline{1-4}
5&GAT&no&GAT \\
\cline{1-4}
6&GAT&yes&GS-GAT \\
\hline
\multicolumn{4}{l}{$^{\mathrm{a}}$al is abbreviation of algorithm and baseline al is DGN.}
\end{tabular}
\label{tab1}
\end{center}
\end{table}

For ablation study, we compare DGN with GS-DGN (DGN with Gumbel Sinkhorn Network), compare GCN with GS-GCN (GCN with Gumbel Sinkhorn Network) and compare GAT with GS-GAT (GAT with Gumbel Sinkhorn Network) to validate the latent permutations learnt by Gumbel Sinkhorn network promote the performance of graph reinforcement learning. Moreover, to ensure the comparison is fair, their basic hyperparameters are all the same and their parameter sizes are also similar.

\subsection{Gather Game}
In this scenario, 74 agents learn to seek and eat 157 food. Gather game scene and its reward mechanism are shown in Fig. \ref{fig4}. In detail, agents have the observation space which is $15\times 15$ map with some channels. This gives one certain agent a local view of $15 \times 15$ map of food, which is used to construct the connectivity of graph. Agents are only rewarded for eating food, which needs to be broken down by 5 attacks before it is absorbed.
Since there is finite food on the map, there is competitive pressure between agents over the food. Agents coordinating by not attacking each other until food is scarce is expected to be seen. When food is scarce, agents may attack each other to try to monopolize the food. Agents can kill each other with a single attack.
Given the situation of broking total number of nodes in graph because of death of agents, SuperSuit, a tool kit which introduces a collection of small functions, is used to wrap MAgent environments to do preprocessing. Function ``black\_death\_v2(env)" is called to make dead agents invisible for other agents, which means no link connected to them and their observations are set to zeros. When an episode game has been played over, the death of agents will be counted and life-death ratio is computed by dividing the life of agents with the death of agents.
\begin{figure}[htbp]
\centerline{\includegraphics[width=7.448cm, height=3.482cm]{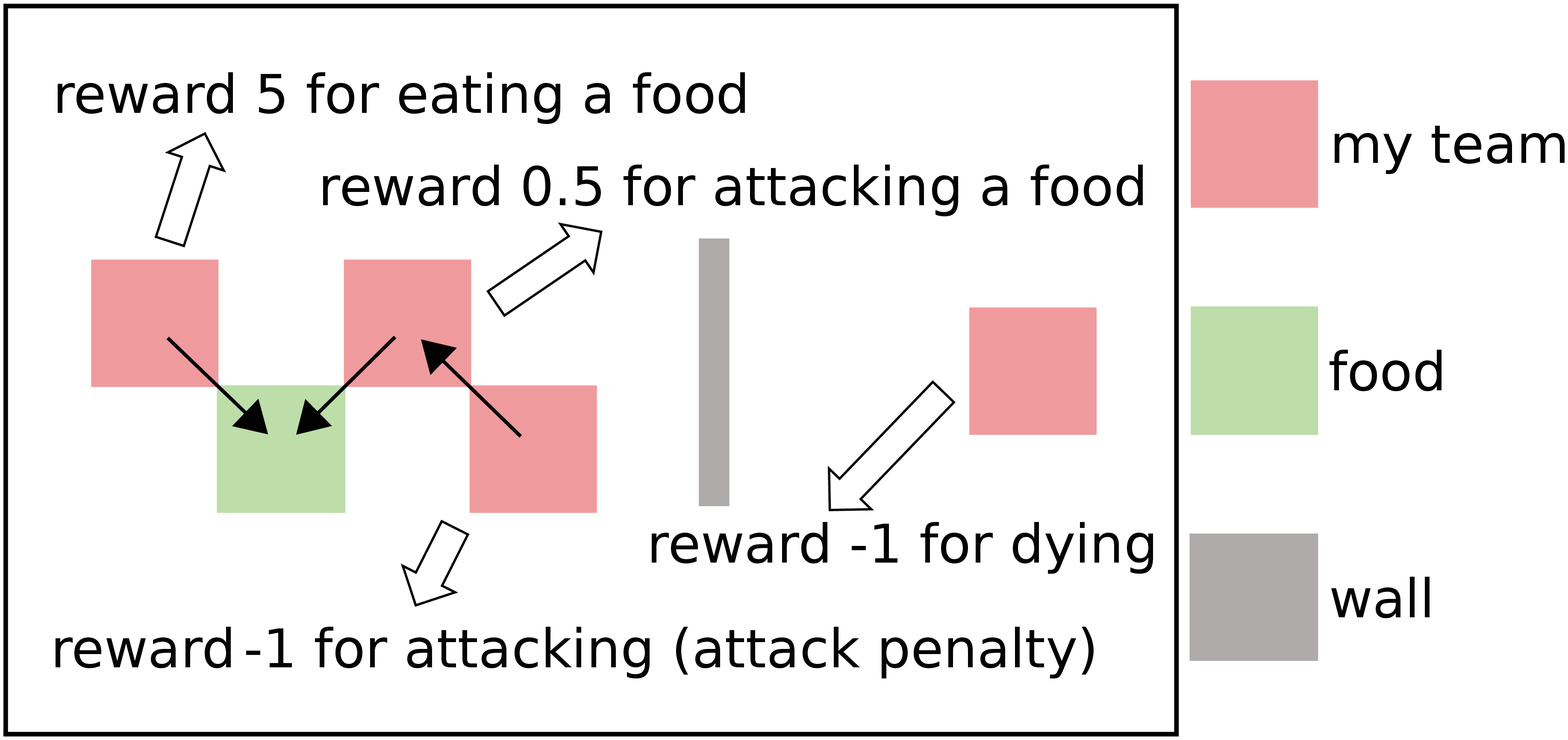}}
\caption{Gather game scene and its reward mechanism.}
\label{fig4}
\end{figure}

\textbf{Result:} Six models listed in TABLE \ref{tab1} are compared together. All models are trained for 511 episodes with the set learning rate 0.001. Epsilon Greedy Strategy is adopted to choose action. Epsilon value $\epsilon$ is set to 0.9 but decays to 0.02 at the rate of 0.05 per episode from the 60th episode. The models will not be trained until the 44th episode. Learning curves in Fig. \ref{fig5} display mean reward gained by all models. Each model experiences three training runs, whose min and max value enclose the shadowed area. Line in middle of shadowed area is the mean value.
\begin{figure}[htbp]
\centerline{\includegraphics[width=8.676cm, height=6.354cm]{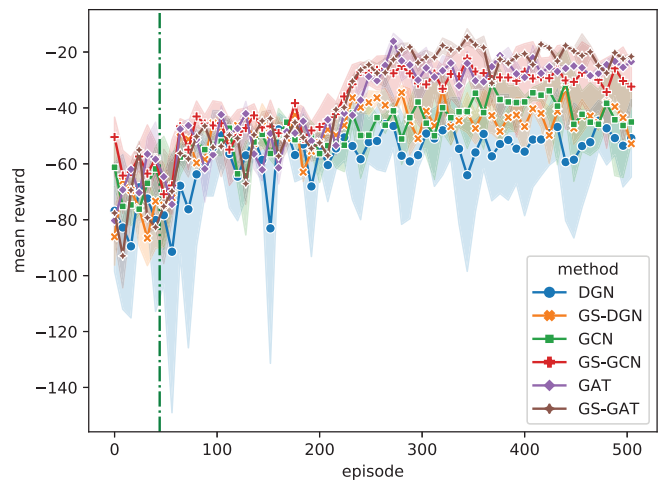}}
\caption{Learning curves of six algorithms.}
\label{fig5}
\end{figure}

Integrated by TABLE \ref{tab2}, summaries are as follows. All algorithms have stable convergence and relatively high rewards, which means that graph neural networks are effectively introduced into reinforcement learning. Similarity of rewards and live-death ratio of GCN and those of DGN shows certain simplicity of this scenario. Rewards gained by GAT are approximate to GS-GCN, which implies that Gumbel Sinkhorn network can improve the expression ability of the model to a certain extent. Nevertheless, the higher live-death ratio of GAT than that of GS-GCN reveals the puissant ability of expression of GAT. Finally, Best performer in Gather game is GS-GAT proposed by this paper with the most rewards and the highest live-death ratio.

\begin{table}[htbp]
\caption{Performance of algorithms in Gather game}
\begin{center}
\begin{tabular}{c|c|c|c|c}
\hline
\textbf{algorithms}&\textbf{mean reward}&\textbf{live}&\textbf{death}&\textbf{live-death ratio} \\
\hline
DGN&-53.78&27&47&0.57 \\
\cline{1-5}
GS-DGN&-42.60&30&44&0.68 \\
\cline{1-5}
GCN&-41.01&30&44&0.68 \\
\cline{1-5}
GS-GCN&-28.92&46&28&1.64 \\
\cline{1-5}
GAT&-25.86&51&23&2.22 \\
\cline{1-5}
GS-GAT&-21.45&56&18&3.11 \\
\hline
\end{tabular}
\label{tab2}
\end{center}
\end{table}

\textbf{Ablation:} To further validate the significance of introducing Gumbel Sinkhorn network, ablation experiments are conducted by comparing DGN and GS-DGN, GCN and GS-GCN, GAT and GS-GAT, whose results are illustrated in Fig. \ref{fig6}.\ref{fig7}.\ref{fig8}.

\begin{figure}[htbp]
\centerline{\includegraphics[width=8.676cm, height=6.354cm]{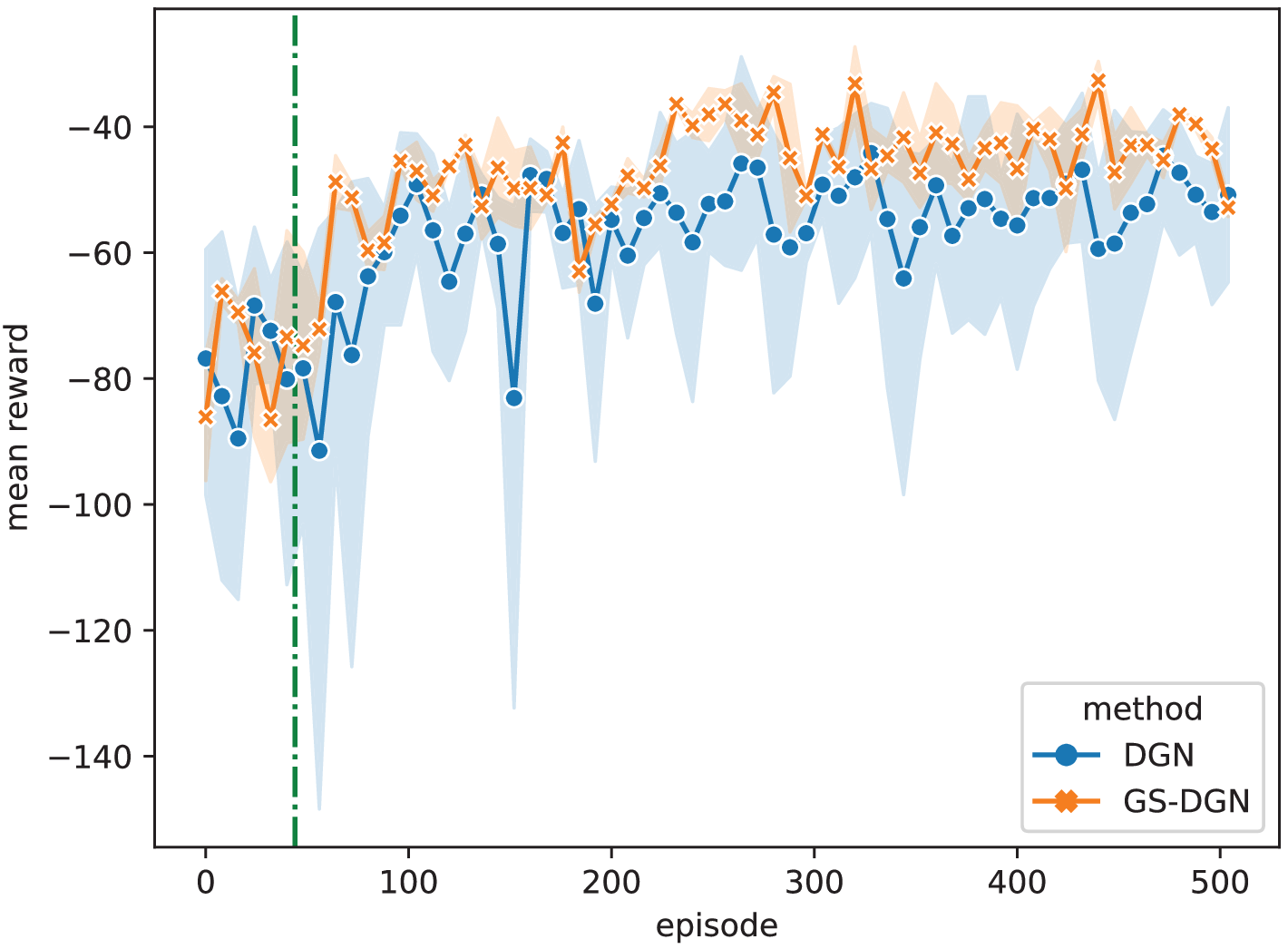}}
\caption{Comparison of learning curves of GS-DGN and DGN.}
\label{fig6}
\end{figure}
\begin{figure}[htbp]
\centerline{\includegraphics[width=8.676cm, height=6.354cm]{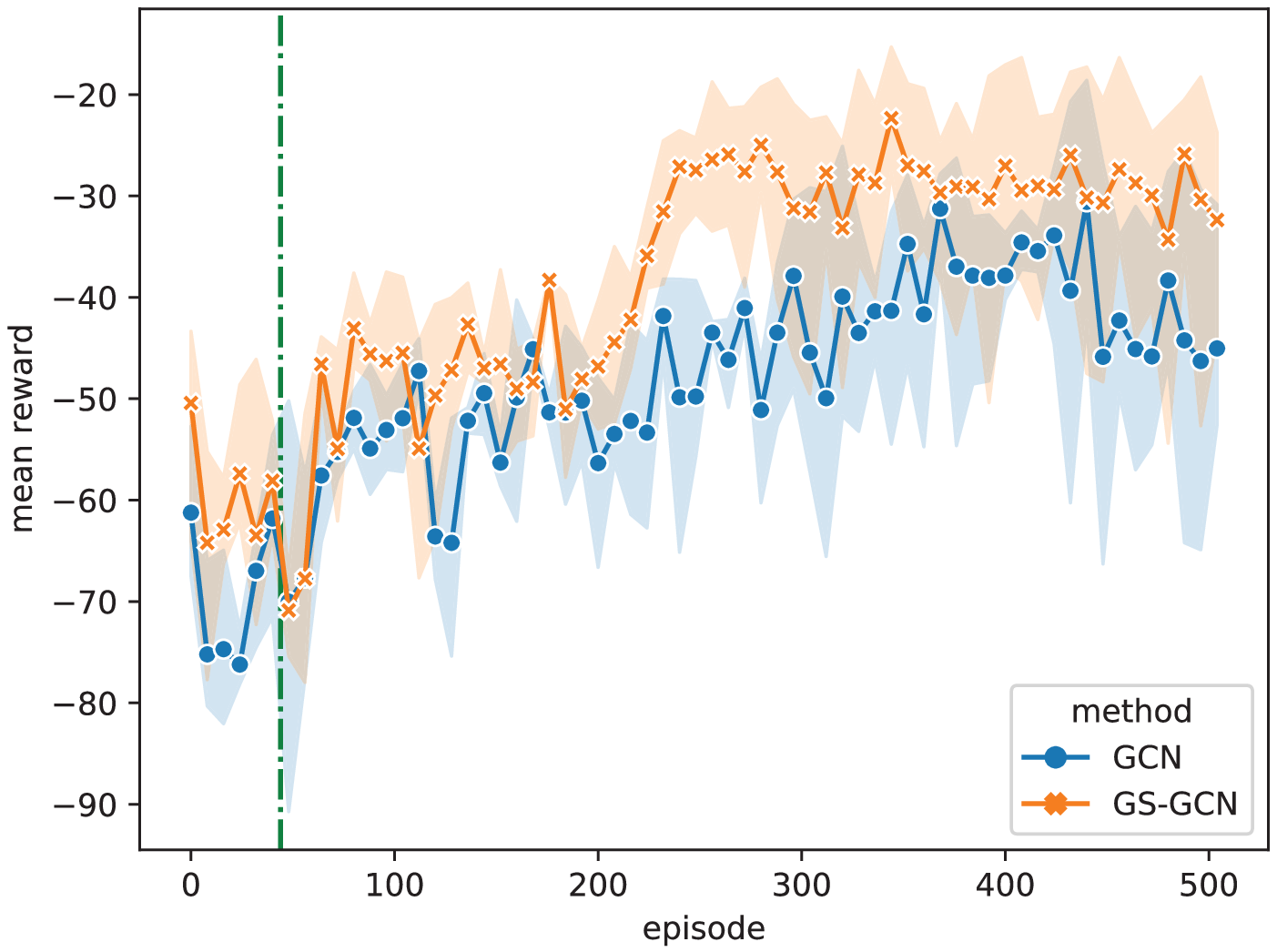}}
\caption{Comparison of learning curves of GS-GCN and GCN.}
\label{fig7}
\end{figure}
\begin{figure}[htbp]
\centerline{\includegraphics[width=8.676cm, height=6.354cm]{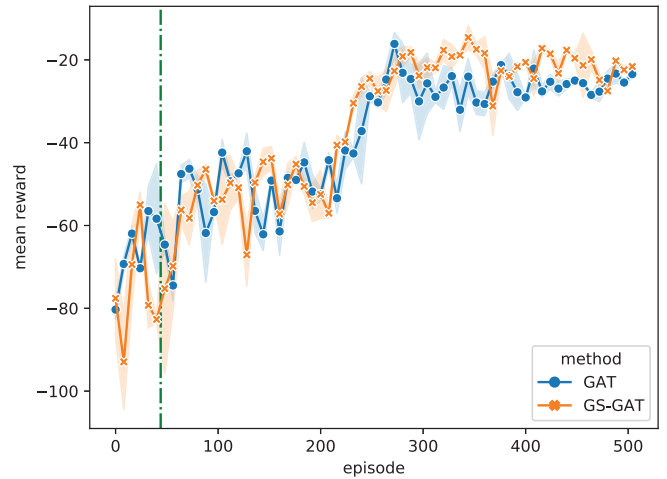}}
\caption{Comparison of learning curves of GS-GAT and GAT.}
\label{fig8}
\end{figure}

Illustration indicates the high rewards gained by GS-DGN than that of DGN. Similar results appear in the comparison of GS-GCN and GCN and that of GS-GAT and GAT. It can be concluded that the latent permutations learnt by Gumbel Sinkhorn network improves the performance of graph reinforcement learning in multi-agent environments.

\subsection{Battle Game}
In this scenario, 30 agents learn to fight against other 30 agents, where the latter team is seen as the enemy of the former team. Both the team have the same abilities. Battle game scene and its reward mechanism are shown in Fig. \ref{fig9}. In detail, they all have the observation space which is $30\times 30$ map with some channels. This gives one certain agent a local view of $30 \times 30$ map of its teammates and enemies, which is used to construct the connectivity of graph. Agents are rewarded for their individual performance, and not for the performance of their neighbors, therefore coordination is difficult for them. Agents slowly regain HP over time, so it is best to kill an opposing agent quickly. Specifically, agents have 10 HP, are damaged 2 HP by each attack, and recover 0.1 HP every turn. Given the situation of broking total number of nodes in graph because of death of agents, SuperSuit, a tool kit which introduces a collection of small functions, is used to wrap MAgent environments to do preprocessing. Function ``black\_death\_v2(env)" is called to make dead agents invisible for other agents, which means no link connected to them and their observations are set to zeros. When an episode game has been played over, the death of both team will be counted and kill-death ratio is computed by dividing the death of other/enemy team with the death of my team. The PPO model trained by stable-baselines3 takes the role of enemy team.
\begin{figure}[htbp]
\centerline{\includegraphics[width=7.448cm, height=3.482cm]{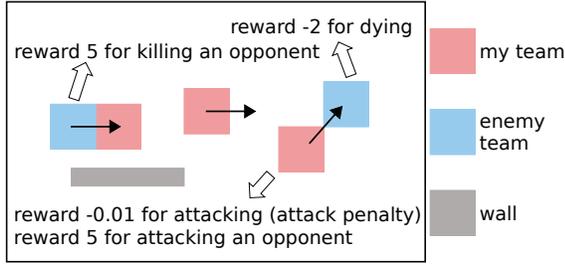}}
\caption{Battle game scene and its reward mechanism.}
\label{fig9}
\end{figure}

\textbf{Result:} Six models listed in TABLE \ref{tab1} are compared together. All models are trained for 511 episodes with the set learning rate 0.001. Epsilon Greedy Strategy is adopted to choose action. Epsilon value $\epsilon$ is set to 0.9 but decays to 0.02 at the rate of 0.05 per episode from the 60th episode. The models will not be trained until the 44th episode. Learning curves in Fig. \ref{fig10} display mean reward gained by all models. Each model experiences three training runs, whose min and max value enclose the shadowed area. Line in middle of shadowed area is the mean value.

\begin{figure}[htbp]
\centerline{\includegraphics[width=8.676cm, height=6.354cm]{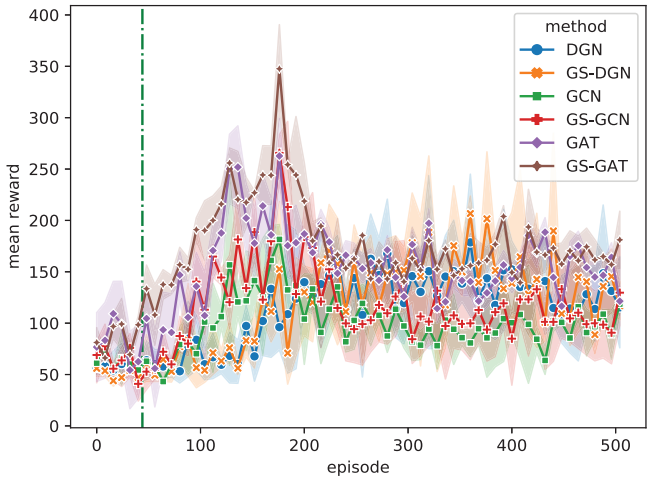}}
\caption{Learning curves of six algorithms.}
\label{fig10}
\end{figure}

Integrated by TABLE \ref{tab3}, summaries are as follows. All algorithms firstly obtain high rewards when begin training but subsequently being poor-performing, which implies the difficulty of this scenario. Consequently they all obtain suboptimal stable convergence but also relatively high rewards which mean that graph neural networks are still effectively introduced into reinforcement learning. Rewards gained by DGN are higher than that gained by GCN and GS-GCN, which means DGN has more comprehensive expression ability in environment with more challenges. Rewards gained by GAT are approximate to GS-DGN and kill-death ratio of GAT is higher than that of GS-DGN. GAT still exhibits the preponderance of adoption of Gumbel Sinkhorn network. GS-GAT achieves more rewards than those of GS-GCN although their same kill-death ratio. Conclusively, Best performer in Battle game is GS-GAT proposed by this paper with the most rewards and the highest kill-death ratio.

\begin{table}[htbp]
\caption{Performance of algorithms in Battle game}
\begin{center}
\begin{tabular}{c|c|c|c|c}
\hline
\textbf{algorithms}&\textbf{mean reward}&\textbf{kill}&\textbf{death}&\textbf{kill-death ratio} \\
\hline
DGN&135.63&10&2&5.00 \\
\cline{1-5}
GS-DGN&148.15&11&2&5.50 \\
\cline{1-5}
GCN&93.18&17&10&1.70 \\
\cline{1-5}
GS-GCN&106.68&11&1&11.00 \\
\cline{1-5}
GAT&152.06&6&1&6.00 \\
\cline{1-5}
GS-GAT&163.95&11&1&11.00 \\
\hline
\end{tabular}
\label{tab3}
\end{center}
\end{table}

\textbf{Ablation:} To further validate the significance of introducing Gumbel Sinkhorn network, ablation experiments are conducted by comparing DGN and GS-DGN, GCN and GS-GCN, GAT and GS-GAT, whose results are illustrated in Fig. \ref{fig11}.\ref{fig12}.\ref{fig13}.

\begin{figure}[htbp]
\centerline{\includegraphics[width=8.676cm, height=6.354cm]{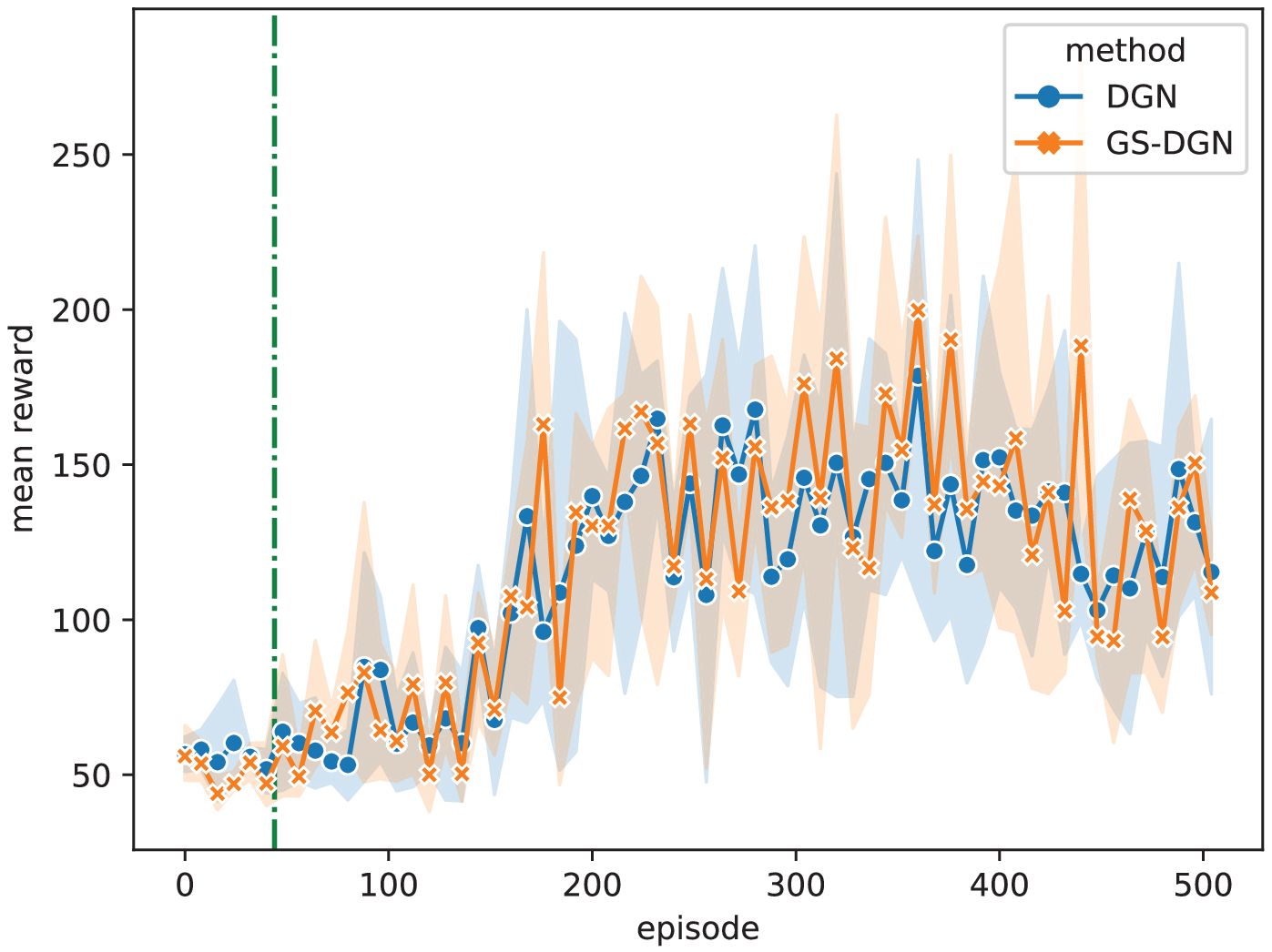}}
\caption{Comparison of learning curves of GS-DGN and DGN.}
\label{fig11}
\end{figure}
\begin{figure}[htbp]
\centerline{\includegraphics[width=8.676cm, height=6.354cm]{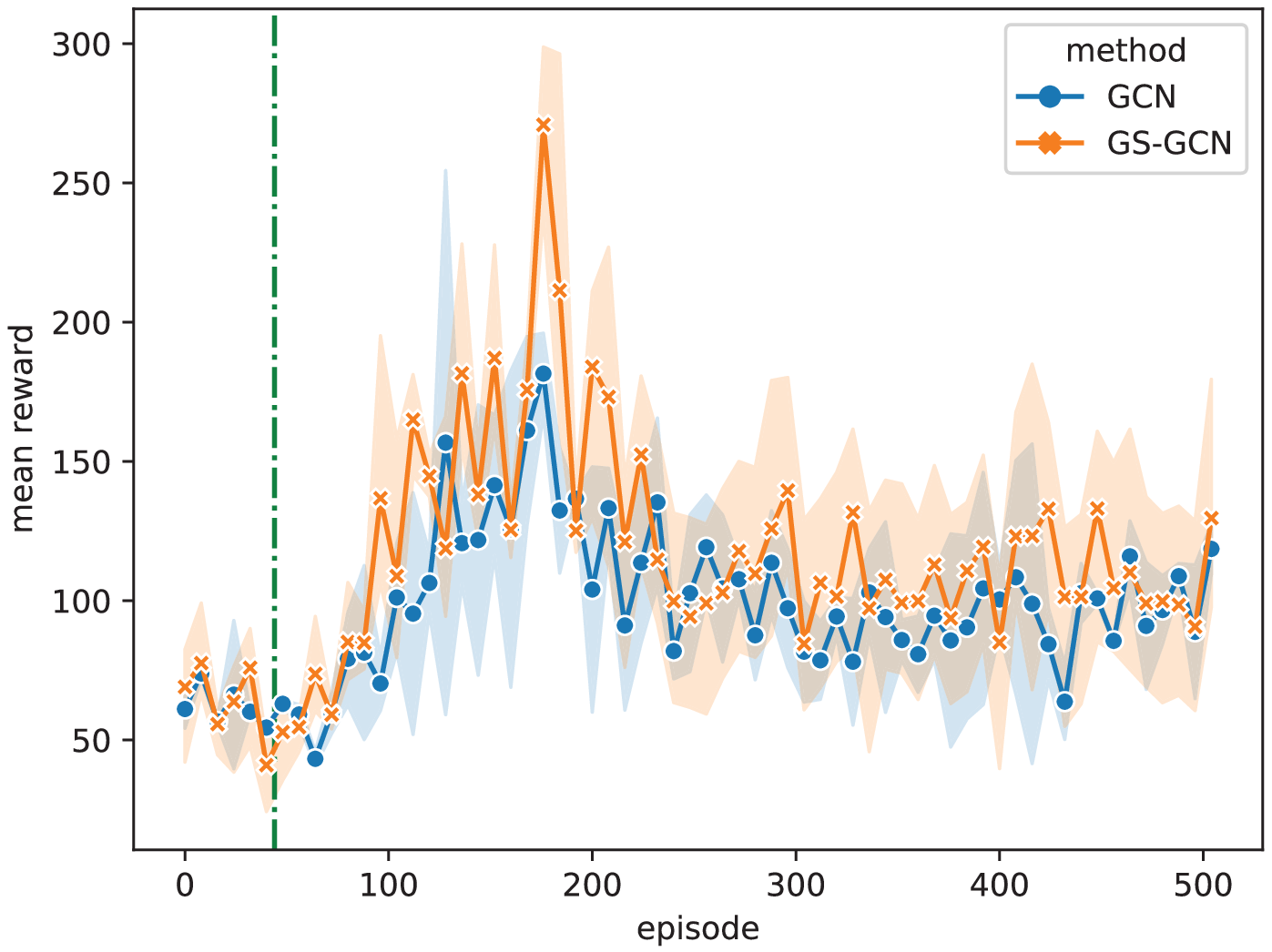}}
\caption{Comparison of learning curves of GS-GCN and GCN.}
\label{fig12}
\end{figure}
\begin{figure}[htbp]
\centerline{\includegraphics[width=8.676cm, height=6.354cm]{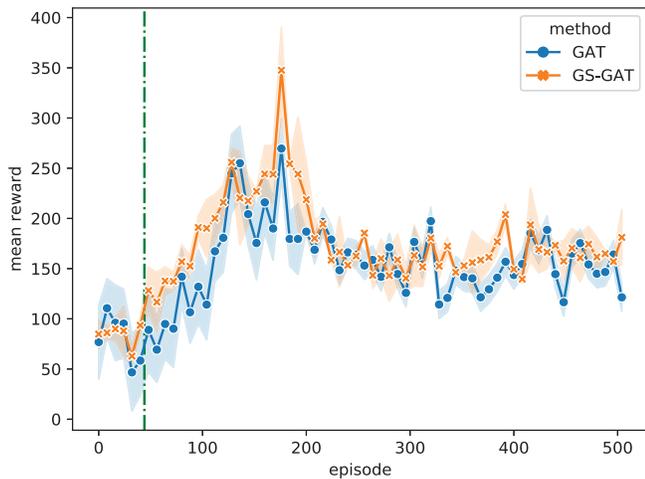}}
\caption{Comparison of learning curves of GS-GAT and GAT.}
\label{fig13}
\end{figure}

Illustration indicates the high rewards gained by GS-DGN than that of DGN. Similar results appear in the comparison of GS-GCN and GCN and that of GS-GAT and GAT. It can also be concluded that the latent permutations learnt by Gumbel Sinkhorn network improves the performance of graph reinforcement learning in multi-agent environments.

\section{Conclusion}
This paper has proposed a gumbel sinkhorn graph attention reinforcement learning. GS-GAT utilizes graph attention network to highly express and deeply learn graph topology structure abstracted from multi-agent environment. Moreover, Gumbel Sinkhorn network which can learn latent permutations is taken advantage of to predict graph topology structure at next timestep. This prediction attaches foresight to agents when cooperating. Empirically, GS-GAT significantly outperforms existing graph reinforcement learning methods in a variety of cooperative multi-agent scenarios.

\section*{Acknowledgment}
This work was supported by the National Natural Science Foundation of China (Nos. 61872019,61972015 and 62002005).

\bibliographystyle{IEEEtran}
\bibliography{IEEEabrv,reference}
\end{document}